\documentclass[a4paper,11pt,twocolumn]{article}
\usepackage{geometry}
\usepackage{graphicx}
\usepackage{authblk}
\usepackage{array}
\usepackage{varwidth}
\usepackage{graphicx}
\usepackage{amsmath}
\usepackage{amssymb}
\usepackage{amsfonts}
\usepackage{amsthm}
\usepackage{algorithm}
\usepackage{algpseudocode}
\usepackage{epsfig}
\usepackage{tensor}
\usepackage{booktabs}
\usepackage{wasysym}
\usepackage{bm}
\usepackage[dvipsnames]{xcolor}
\usepackage{url}
\usepackage{rotating}
\usepackage{lineno,hyperref}
\usepackage[font=small,labelfont=bf]{caption}

\newcommand\T{{\hspace{-0pt}\intercal}}
\newcommand*\diff{\mathop{}\!\mathrm{d}}

\DeclareMathOperator{\sat}{sat}

\DeclareMathOperator*{\argmin}{arg\,min}

\DeclareMathOperator{\Diag}{diag}

\newtheorem{proposition}{Proposition}

\newtheorem{remark}{Remark}

\DeclareGraphicsExtensions{.pdf}

\begin{document}

\title{A Lyapunov-Stable Adaptive Method to Approximate Sensorimotor Models for Sensor-Based Control}
\author[a]{David Navarro-Alarcon}
\author[a]{Jiaming Qi}
\author[b]{Jihong Zhu}
\author[b]{Andrea Cherubini}
\affil[a]{The Hong Kong Polytechnic University, Hong Kong}
\affil[b]{Universit\'{e} de Montpellier / LIRMM, France}
\date{}

\maketitle

\begin{abstract}
In this article, we present a new scheme that approximates unknown sensorimotor models of robots by using feedback signals only.
The formulation of the uncalibrated sensor-based regulation problem is first formulated, then, we develop a computational method that distributes the model estimation problem amongst multiple adaptive units that specialise in a local sensorimotor map.
Different from traditional estimation algorithms, the proposed method requires little data to train and constrain it (the number of required data points can be analytically determined) and has rigorous stability properties (the conditions to satisfy Lyapunov stability are derived).
Numerical simulations and experimental results are presented to validate the proposed method.
\end{abstract}


\section{Introduction}
	Robots are widely used in industry to perform a myriad of sensor-based applications ranging from visually servoed pick-and-place tasks to force-regulated workpiece assemblies \cite{Book:Nof1999}.
	Their accurate operation is largely due to the fact that industrial robots rely on fixed settings that enable the exact characterisation of the tasks' sensorimotor model.
	Although this full characterisation requirement is fairly acceptable in industrial environments, it is too stringent for many service applications where the mechanical, perceptual and environment conditions are not exactly known or might suddenly change \cite{dna2019_ccc}, e.g. in domestic robotics (where environments are highly dynamic), field robotics (where variable morphologies are needed to navigate complex workspaces), autonomous systems (where robots must adapt and operate after malfunctions), to name a few cases.
	
	In contrast to industrial robots, the human brain has a high degree of adaptability that allows it to continuously learn sensorimotor relations.
	The brain can seemingly coordinate the body (whose morphology persistently changes throughout life) under multiple circumstances: severe injuries, amputations, manipulating tools, using prosthetics, etc.
	It can also recalibrate corrupted or modified perceptual systems: a classical example is the manipulation experiment performed in \cite{Journals:Kohler1962} with image inverting goggles that altered a subject's visual system. 
	In infants, motor babbling is used for obtaining (partly from scratch and partly innate) a coarse sensorimotor model that is gradually refined with repetitions \cite{Journals:von1982}.
	Providing robots with similar incremental and life-long adaptation capabilities is precisely our goal in this paper.

	From an automatic control point of view, a sensorimotor model is needed for coordinating input motions of a mechanism with output sensor signals \cite{Journals:Huang1994}, e.g. controlling the shape of a manipulated soft object based on vision \cite{dna2016_tro} or controlling the balance of a walking machine based on a gyroscope \cite{Proceedings:Yu2018}. 
	In the visual servoing literature, the model is typically represented by the so-called interaction matrix \cite{Journals:Cherubini2015,Journals:Hutchinson1996}, which is computed based on kinematic relations between the robot's configuration and the camera's image projections.
	In the general case, sensorimotor models depend on the physics involved in constructing the output sensory signal; If this information is uncertain (e.g. due to bending of robot links, repositioning of external sensors, deformation of objects), the robot may no longer properly coordinate actions with perception.
	Therefore, it is important to develop methods that can efficiently provide robots with the capability to adapt to unforeseen changes of the sensorimotor conditions.

	Classical methods in robotics to compute this model (see \cite{Journals:Sigaud2011} for a review) can be roughly classified into \emph{structure-based} and \emph{structure-free} approaches \cite{dna2019_ccc}.
	The former category represents ``calibration-like'' techniques (e.g. off-line \cite{Journals:Wei1998} or adaptive \cite{dna2015_iros,Journals:Liu2013,Journals:Wang2008}) that aim to identify the unknown model parameters.
	These approaches are easy to implement, however, they require exact knowledge of the analytical structure of the sensory signal (which might not be available or subject to large uncertainties).
	Also, since the resulting model is fixed to the mechanical/perceptual/environmental setup that was used for computing it, these methods are not robust to unforeseen changes. 
	
	For the latter (structure-free) category, we can further distinguish between two main types \cite{dna2019_ccc}: instantaneous and distributed estimation.
	The first type performs online numerical approximations of the unknown model (whose structure does not need to be known); Some common implementations include e.g. Broyden-like methods \cite{Journals:Alambeigi2018,Proceedings:Jagersand1997,Proceedings:Hosoda1994} and iterative gradient descent rules \cite{tiffany2017_rcar,dna2015_iros}.
	These methods are robust to sudden configuration changes, yet, as the sensorimotor mappings are continuously updated, they do not preserve knowledge of previous estimations (i.e. it's model is only valid for the current local configuration).
	The second type distributes the estimation problem amongst multiple computing units; The most common implementation is based on (highly nonlinear) connectionists architectures \cite{Journals:Hu2019,Proceedings:Lyu2018,Journals:Li2014}. 
	These approaches require very large amounts of training data to properly constrain the learning algorithm, which is impractical in many situations.
	Other distributed implementations (based on SOM-like sensorimotor ``patches'' \cite{Journals:Kohonen2013}) are reported e.g. in \cite{omar2019_taros,Journals:Pierris2017,Journals:Escobar2016}, yet, the stability properties of its algorithms are not rigorously analysed.

	As a solution to these issues, in this paper we propose a new approach that approximates unknown sensorimotor models based on local data observations only.
	In contrast to previous state-of-the-art methods, our adaptive algorithm has the following original features: 
	\begin{itemize}
		\item It requires few data observations to train and constrain the algorithm (which allows to implement it in real-time).
		\item The number of minimum data points to train it can be analytically obtained (which makes data collection more effective).
		\item The stability of its update rule can be rigorously proved (which enables to deterministically predict its performance).
	\end{itemize}
	The proposed method is general enough to be used with different types of sensor signals and robot mechanisms.

	The rest of the manuscript is organised as follows: Sec. \ref{sec:preliminaries} presents preliminaries, Sec. \ref{sec:methods} describes the proposed method, Sec. \ref{sec:results} reports the conducted numerical study and Sec. \ref{sec:conclusion} gives final conclusions.

	\section{Preliminaries} \label{sec:preliminaries}
	
	\subsection{Notation}
	Along this note we use very standard notation.
	Column vectors are denoted with bold small letters $\mathbf m$ and matrices with bold capital letters $\mathbf M$. 
	Time evolving variables are represented as $\mathbf m_t$, where the subscript $\ast_t$ denotes the discrete time instant. 
	Gradients of functions $b=\beta(\mathbf m):\mathcal M\mapsto\mathcal B$ are denoted as $\nabla\beta(\mathbf m) = (\partial \beta/\partial\mathbf m{)}^\T$.
	
	\subsection{Configuration Dependant Feedback}
	Consider a fully-actuated robotic system whose instantaneous configuration vector (modelling e.g. end-effector positions in a manipulator, orientation in a robot head, etc.) is denoted by the vector $\mathbf x_t\in\mathbb R^n$.
	Such model can only be used to represent traditional \emph{rigid} systems, thus, it excludes soft/continuum mechanisms \cite{Journals:Falkenhahn2015} or robots driven by elastic actuators \cite{zerui2015_tmech}.
	Without loss of generality, we assume that its coordinates are all represented using the same unitless range\footnote{This can be easily obtained with constant kinematic transformations.}.
	To perform a task, the robot is equipped with a sensing system that continuously measure a physical quantity whose instantaneous values depend on $\mathbf x_t$. 
	Some examples of these types of configuration-dependent feedback signals are: geometric features in an image \cite{marie2020_ral}, forces applied onto a compliant surface \cite{Journals:Bouyarmane2019}, proximity to an object \cite{Journals:Cherubini2013}, intensity of an audio source \cite{Proceedings:Magassouba2016}, attitude of a balancing body \cite{Journals:Defoort2009}, shape of a manipulated object \cite{dna2018_tro}, temperature from a heat source \cite{Proceedings:Saponaro2015}, etc.
	
	Let $\mathbf y_t\in\mathbb R^m$ denote the vector of feedback features that quantify the task; Its coordinates might be constructed with raw measurements or be the result of some processing.
	We model the instantaneous relation between this sensor signal and the robot's configuration as \cite{Journals:Chaumette2006}:
	\begin{equation}
	\mathbf y_t = f(\mathbf x_t) : \mathbb R^n \mapsto \mathbb R^m
	\label{eq:sensor_model}
	\end{equation}

	\begin{remark}
	Along this paper, we assume that the feedback feature functional $f(\mathbf x_t)$ is smooth (at least twice differentiable) and its Jacobian matrix has a full row/column rank (which guarantees the existence of its (pseudo-)inverse).
	\end{remark}

	\subsection{Uncalibrated Sensorimotor Control}
	In our formulation of the problem, it is assumed that the robotic system is controlled via a standard position/velocity interface, as in e.g. \cite{Journals:Whitney1969,Journals:Siciliano1990}, a situation that closely models the majority of commercial robots.
	With position interfaces, the motor action $\mathbf u_t\in\mathbb R^n$ represents the following displacement difference:
	\begin{equation}
	\mathbf x_{t+1} - \mathbf x_t = \mathbf u_t	
	\label{eq:displacement_control}
	\end{equation}
	Such \emph{kinematic control} interface renders the typical stiff behaviour present in industrial robots (for this model, external forces do not affect the robot's trajectories).
	The methods in this paper are formulated using position commands, however, these can be easily transformed into robot velocities $\mathbf v_t\in\mathbb R^n$ by dividing $\mathbf u_t$ by the servo controller's time step $\diff t$ as follows $\mathbf u_t/\diff t = \mathbf v_t$.
	
	The expression that describes how the motor actions result in changes of feedback features is represented by the first-order difference model\footnote{This difference equation represents the discrete-time model of the robot's differential sensor kinematics.}:
	\begin{equation}
	\mathbf y_{t+1} = \mathbf y_t + \mathbf A (\mathbf x_t) \mathbf u_t = \mathbf y_t + {\boldsymbol \delta}_t
	\label{eq:differential_model}
	\end{equation}
	where the configuration-dependent matrix $\mathbf A (\mathbf x_t) = \partial f/\partial \mathbf x_t \in \mathbb R^{m\times n}$ represents the traditional sensor Jacobian matrix of the system (also known as the interaction matrix in the visual servoing literature \cite{Journals:Hutchinson1996}).
	To simplify notation, throughout this paper we shall omit its dependency on $\mathbf x_t$ and denote it as $\mathbf A_t = \mathbf A (\mathbf x_t)$.
	The flow vector $\boldsymbol \delta_t = \mathbf A_t \mathbf u_t\in\mathbb R^m$ represents the sensor changes that result from the action $\mathbf u_t$.
	Figure \ref{fig:sensorimotor_mappings} conceptually depicts these quantities.
	
	The sensorimotor control problem consists in computing the necessary motor actions for the robot to achieve a desired sensor configuration.
	Without loss of generality, in this note, such configuration is characterised as the regulation of the feature vector $\mathbf y_t$ towards a constant target $\mathbf y^*$.
	The necessary motor action to reach the target can be computed by minimising the following quadratic cost function:
	\begin{equation}
	J = \left\| \lambda\sat(\mathbf y_t-\mathbf y^*) + \mathbf A_t\mathbf u_t \right\|^2
	\label{eq:cost_function}
	\end{equation}
	where $\lambda>0$ is a gain and $\sat(\cdot)$ a standard saturation function (defined as in e.g. \cite{Journals:Chang2018}).
	The rationale behind the minimisation of the cost \eqref{eq:cost_function} is to find an incremental motor command $\mathbf u_t$ that forward-projects into the sensory space (via the interaction matrix $\mathbf A_t$) as a vector pointing towards the target $\mathbf y^*$. 
	By iteratively commanding these motions, the distance $\|\mathbf y_t-\mathbf y^*\|$ is expected to be asymptotically minimised.

	To obtain $\mathbf u_t$, let us first compute the extremum $\nabla J(\mathbf u_t) = \mathbf 0$, which yields the normal equation
	\begin{equation}
	\mathbf A_t^\T \mathbf A_t\mathbf u_t = - \lambda \mathbf A_t^\T \sat(\mathbf y_t-\mathbf y^*)
	\label{eq:normal_equation}
\end{equation}
	Solving \eqref{eq:normal_equation} for $\mathbf u_t$, gives rise to the motor command that minimises $J$:
	\begin{equation}
	\mathbf u_t = - \lambda \mathbf A^{\#}_t \sat(\mathbf y_t-\mathbf y^*)
	\label{eq:general_motor_action}
	\end{equation}
	where $\mathbf A^{\#}_t \in \mathbb R^{n\times m}$ is a generalised pseudo-inverse matrix satisfying $\mathbf A_t\mathbf A^{\#}_t\mathbf A_t=\mathbf A_t$ \cite{Book:Nakamura1991},
	whose existence is guaranteed as $\mathbf A_t$ has a full column/row rank (depending on whichever is larger $n$ or $m$).
	Yet, note that for the case where $m>n$, the cost function $J$ can only be locally minimised.

	Note that the computation of \eqref{eq:general_motor_action} requires exact knowledge of $\mathbf A_t$.
	To analytically calculate this matrix, we need to \emph{fully calibrate} the system, which is too restrictive for applications where the sensorimotor model is unavailable or might suddenly change.
	This situation may happen if the mechanical structure of the robot is altered (e.g. due to bendings or damage of links), or the configuration of the perceptual system is changed (e.g. due to relocating external sensors), or the geometry of a manipulated object changes (e.g. due to grasping forces deforming a soft body), to name a few cases.
	Without this information, the robot may not properly coordinate actions with perception.
	In the following section, we describe our proposed solution.
		
	\begin{figure}[t]
		\centering
		\includegraphics[width = \columnwidth]{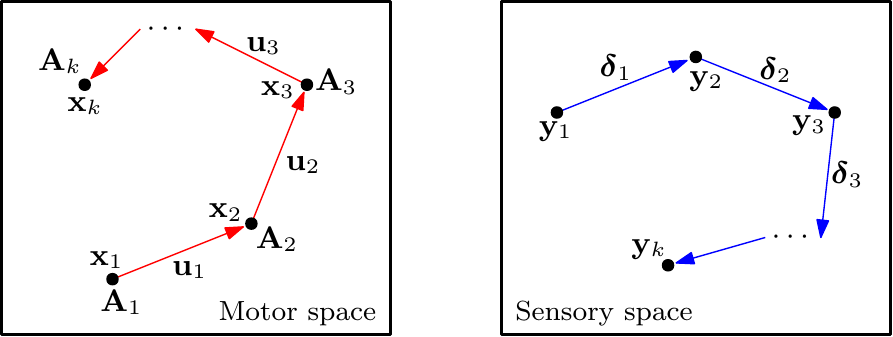}
		\caption{Representation of a configuration trajectory $\mathbf x_t$, its associated transformation matrices $\mathbf A_t$ and motor actions $\mathbf u_t$, that produce the measurements $\mathbf y_t$ and sensory changes $\boldsymbol\delta_t$.}
		\label{fig:sensorimotor_mappings}
	\end{figure}

	\section{Methods}\label{sec:methods}
	\subsection{Discrete Configuration Space}
	Since the (generally non-linear) feature functional \eqref{eq:sensor_model} is smooth, the Jacobian matrix $\mathbf A_t = \partial f/\partial \mathbf x_t$ is also expected to smoothly change along the robot's configuration space.
	This situation means that a \emph{local} estimation $\widehat{\mathbf A}$ of the true matrix $\mathbf A_t$ around a configuration point $\mathbf x_i$ is also valid around the surrounding neighbourhood \cite{Journals:Sang2012}.
	We exploit this simple yet powerful idea to develop a computational method that distributes the model estimation problem amongst various units that specialise in a local sensorimotor map.

	It has been proved in the sensor-based control community \cite{Journals:Cheah2003} that rough estimations of $\mathbf A_t$ (combined with the rectifying action of feedback) are sufficient for guiding the robot with sensory signals.
	However, note that large deviations from such configuration point $\mathbf x_i$ may result in model inaccuracies.
	Therefore, the local neighbourhoods cannot be too large.
	
	Consider a system with $N$ computing units distributed around the robot's configuration space, see Figure \ref{fig:neighbourhood}.
	The location of these units can be defined with many approaches, e.g. with self organisation \cite{Book:Kohonen2001}, random distributions, uniform distributions, etc. \cite{Book:Haykin2009}.
	To each unit, we associate the following 3-tuple:
	\begin{equation}
	z^l = \begin{Bmatrix} \mathbf w^l & \widehat{\mathbf A}_t^l & \mathcal D^l \end{Bmatrix},\quad \text{for}\quad l = 1,\ldots,N
	\end{equation}
	The weight vector $\mathbf w^l\in\mathbb R^n$ represents a configuration $\mathbf x_t$ of the robot where $\mathbf w^l = \mathbf x_t$.
	The matrix $\widehat{\mathbf A}_t^l \in\mathbb R^{n\times m}$ stands for a local approximation of $\mathbf A_t(\mathbf w^l)$ evaluated at the point $\mathbf w^l$.
	The purpose	of the structure $\mathcal D^l$ is to store sensor and motor observations $\mathbf d_t = \{ \mathbf x_t , \mathbf u_t , \boldsymbol\delta_t \}$, that are collected around the vicinity of $\mathbf w^l$ through babbling-like motions \cite{Proceedings:Saegusa2009}.
	The structure $\mathcal D^l$ is constructed as follows:
	\begin{equation}
	\mathcal D^l = \begin{Bmatrix} \mathbf d_1 & \mathbf d_2 & \cdots & \mathbf d_\tau \end{Bmatrix}^\T
	\label{eq:data_structure}
	\end{equation}
	for $\tau>0$ as the total number of observations, which once collected, they remain constant during the learning stage.
	Note that $\mathbf x_i$ and $\mathbf x_{i+1}$ are typically not consecutive time instances.
	The total number $\tau$ of observations is assumed to satisfy $\tau > mn$.
	
	\subsection{Initial Learning Stage}
	We propose an adaptive method to iteratively compute the local transformation matrix from data observations.
	To this end, consider the following quadratic cost function for the $l$th unit:
	\begin{align}
	Q^l &= \frac{1}{2} \sum_{k=1}^{\tau} h^{lk} \left\| \widehat{\mathbf A}_t^l \mathbf u_k - \boldsymbol\delta_k \right\|^2 \nonumber \\
	&= \frac{1}{2} \sum_{k=1}^{\tau} h^{lk} \left\| \mathbf F(\mathbf u_k) \widehat{\mathbf a}^l_t - \boldsymbol\delta_k \right\|^2 
	\label{eq:cost_function_distributed}
	\end{align}
	for $\mathbf F(\mathbf u_k)\in\mathbb R^{m\times mn}$ as a regression-like matrix defined as
	\begin{equation}
	\mathbf F(\mathbf u_k) = 
	\begin{bmatrix}
	\mathbf u_k^\T & \mathbf 0_n^\T & \cdots & \mathbf 0_n^\T \\
	\mathbf 0_n^\T & \mathbf u_k^\T & \cdots & \mathbf 0_n^\T \\
	\vdots & \vdots & \ddots & \vdots \\
	\mathbf 0_n^\T & \mathbf 0_n^\T & \cdots & \mathbf u_k^\T
	\end{bmatrix}
	\end{equation}
	and a vector of adaptive parameters $\widehat{\mathbf a}^l_t\in\mathbb R^{nm}$ constructed as:
	\begin{equation}
	\widehat{\mathbf a}^l_t = 
	\begin{bmatrix}
	\hat{a}^{l11}_t & \hat{a}^{l12}_t & \cdots & \hat{a}^{lmn}_t 
	\end{bmatrix}^\T
	\end{equation}
	where the scalar $\hat{a}^{lij}_t$ denotes the $i$th row $j$th column element of the matrix $\widehat{\mathbf A}^l_t$.
	
	The scalar $h^{lk}$ represents a Gaussian neighbourhood function centred at the $l$th unit and computed as:
	\begin{equation}
	h^{lk} = \exp \left( - \frac{ \|{\mathbf w}^l - {\mathbf x}_k \|^2 }{2\sigma^2} \right)
	\end{equation}
	where $\sigma>0$ (representing the standard deviation) is used to control the width of the neighbourhood.
	By using $h^{lk}$, the observations' contribution to the cost \eqref{eq:cost_function_distributed} proportionally decreases with the distance to $\mathbf w^l$.
	The dimension of the neighbourhood is defined such that $h\approx 0$ is never satisfied for any of its observations $\mathbf x^k$.
	In practice, it is common to approximate the Gaussian shape with a simple ``square'' region, which presents the highest approximation error around its corners (see e.g. Figure \ref{fig:neighbourhood} where the sampling point $\mathbf d_{\tau+1}$ is within its boundary).

	\begin{figure}[t]
		\centering
		\includegraphics[width = \columnwidth]{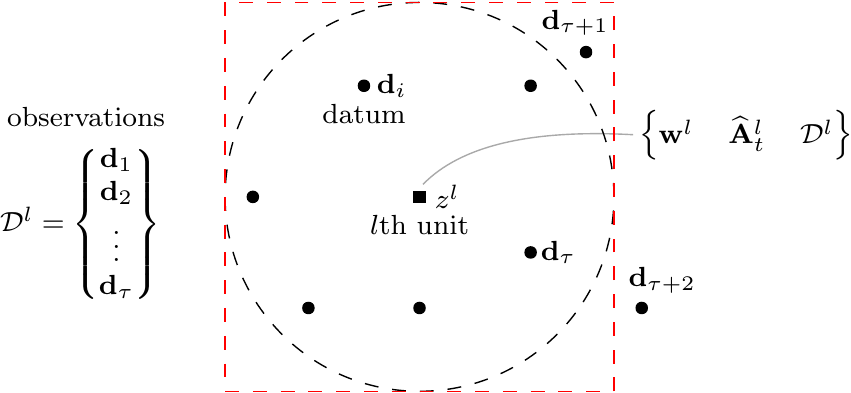}
		\caption{Representation of the $l$th computing unit and the neighbouring data used to approximate the local sensorimotor model. The black and red dashed depict the Gaussian and its square approximation.}
		\label{fig:neighbourhood}
	\end{figure}

	To compute an accurate sensorimotor model, the data points in \eqref{eq:data_structure} should be as distinctive as possible (i.e. the motor observations $\mathbf u_t$ should not be collinear). 
	This requirement can be fairly achieved by covering the uncertain configuration with curved/random motions.
	
	The following gradient descent rule is used for approximating the transformation matrix $\mathbf A_t$ at the $l$th unit:
	\begin{align}
	\widehat{\mathbf a}^{l}_{t+1} &= \widehat{\mathbf a}^{l}_t - \gamma \nabla Q^l(\widehat{\mathbf a}^{l}_t) \label{eq:update_rule} \\
	&= \widehat{\mathbf a}^{l}_t - \gamma \sum_{k=1}^{\tau} h^{lk} \mathbf F(\mathbf u_k)^\T \left( \widehat{\mathbf A}_t^l \mathbf u_k - \boldsymbol\delta_k \right) \nonumber
	\end{align}
	for $\gamma>0$ as a positive learning gain.
	For ease of implementation, the update rule \eqref{eq:update_rule} can be equivalently expressed in scalar form as:
	\begin{equation}
	\hat{a}^{lij}_{t+1} = \hat{a}^{lij}_t - \gamma \sum_{k=1}^{\tau} h^{lk} u_k^j \left\{ \left( \sum_{r=1}^{n} \hat{a}_t^{lir} u_k^r  \right) - \delta_k^i \right\}
	\label{eq:updae_rule_scalar}
	\end{equation}
	where $u_k^j$ and $\delta_k^i$ denote the $j$th and $i$th components of the vectors $\mathbf u_k$ and $\boldsymbol\delta_k$, respectively.
	
	\begin{remark}
	There are other estimation methods in the literature that also make use of Gaussian functions, e.g. radial basis functions (RBF) \cite{Journals:Li2014} to name an instance. 
	However, RBF (in its standard formulation) use configuration-dependent Gaussians to modulate a set of weights (which provide non-linear approximation capabilities), whereas in our case, the Gaussians are used but within the weights' adaptation law to proportionally scale the contribution of the collected sensory-motor data (our method provides a \emph{linear} approximation within the neighbourhood).
	Our Gaussian weighted approach most closely resembles the one used in self organising maps (SOM) \cite{Journals:Kohonen2013} to combine surrounding data observations. 
	\end{remark}

	\subsection{Lyapunov Stability}
	In this section, we analyse the stability properties of the proposed update rule by using discrete-time Lyapunov theory \cite{Preprint:Bof2018}.
	To this end, let us first assume that the transformation matrix satisfies: 
	\begin{equation}
	\mathbf A(\mathbf w^l) = \partial f \ \partial \mathbf x (\mathbf w^l) \approx \mathbf A(\mathbf x_j)
	\end{equation} 
	for any configuration $\mathbf x_j$ around the neighbourhood defined by $\mathcal D^l$ (this situation implies that $\mathbf A(\cdot)$ is constant around the vicinity of $\mathbf w^l$). 
	Therefore, we can locally express around $\mathbf w^l$ the sensor changes as: 
	\begin{equation}
	\boldsymbol\delta_k = \mathbf F(\mathbf u_k)\mathbf a^l
	\end{equation}
	where $\mathbf a^l= [a^{l11},{a}^{l12},\ldots,{a}^{lmn}{]}^\T\in\mathbb R^{mn}$ denotes the vector of \emph{constant} parameters, for ${a}^{lij}$ as the $i$th row $j$th column of the unknown matrix ${\mathbf A}(\mathbf w^l)$.
	To simplify notation, we shall denote $\mathbf F_k = \mathbf F(\mathbf u_k)$.

	\begin{proposition}
		For a number $mn$ of linearly independent vectors $\mathbf u_k$, the adaptive update rule \eqref{eq:update_rule} asymptotically minimises the magnitude of the parameter estimation error $\|\widehat{\mathbf a}_t^l - \mathbf a^l\|$.
	\end{proposition}
	
	\begin{proof}
		Consider the following quadratic (energy-like) function:
		\begin{equation}
		V_t^l = \left\|\widehat{\mathbf a}_t^l - \mathbf a^l \right\|^2
		\label{lyapunov_function}
		\end{equation}
		Computing the forward difference of $V_t^l$ yields:
		\begin{multline}
		V_{t+1}^l - V_t^l
		= \left\|\widehat{\mathbf a}_{t+1}^l - \mathbf a^l \right\|^2 - \left\|\widehat{\mathbf a}_t^l - \mathbf a^l \right\|^2 \nonumber \\
		= \left\| \left[ \mathbf I - \gamma \sum_{k=1}^{\tau} h^{lk} \mathbf F_k^\T \mathbf F_k \right] \left(\widehat{\mathbf a}_t^l - \mathbf a^l \right) \right\|^2 \nonumber \\ 
		- \left\|\widehat{\mathbf a}_t^l - \mathbf a^l \right\|^2 = - \left(\widehat{\mathbf a}_t^l - \mathbf a^l \right)^\T \boldsymbol\Omega \left(\widehat{\mathbf a}_t^l - \mathbf a^l \right)
		\end{multline}
		for a symmetric matrix $\boldsymbol\Omega\in\mathbb R^{mn\times mn}$ defined as follows:
		\begin{align}
		\boldsymbol\Omega 
		&= \mathbf I - \left[\mathbf I - \gamma \sum_{k=1}^{\tau} h^{lk} \mathbf F_k^\T \mathbf F_k \right]^2 \nonumber \\
		&= 2\gamma \sum_{k=1}^{\tau} h^{lk} \mathbf F_k^\T \mathbf F_k - \gamma^2 \left[ \sum_{k=1}^{\tau} h^{lk} \mathbf F_k^\T \mathbf F_k \right]^2 \nonumber \\
		&= \gamma \boldsymbol\Phi^\T \underbrace{ \left( 2\mathbf H - \gamma \mathbf H \boldsymbol\Phi \boldsymbol\Phi^\T \mathbf H  \right) }_{\mathbf C} \boldsymbol\Phi
		\label{eq:dissipation_matrix}
		\end{align}
		with $\mathbf H=\Diag(h^{l1}\mathbf I_{\tau},\ldots,h^{l\tau}\mathbf I_\tau)\in\mathbb R^{m\tau\times m\tau}$ as a positive-definite diagonal matrix, $\mathbf I_\tau\in\mathbb R^{\tau\times\tau}$ as an identity matrix and $\boldsymbol\Phi\in\mathbb R^{m\tau\times mn}$ constructed with $\tau$ matrices $\mathbf F_k$ as follows:
		\begin{equation}
		\boldsymbol\Phi = 
		\begin{bmatrix}
		\mathbf F^\T_1 & \mathbf F_2^\T & \cdots & \mathbf F^\T_\tau
		\end{bmatrix}^\T
		\end{equation}
		To prove the asymptotic stability of \eqref{eq:update_rule}, we must first prove the positive-definiteness of the dissipation-like matrix $\boldsymbol\Omega$ \cite{Book:van_der_Schaft}.
		To this end, note that since the ``tall'' observations' matrix $\boldsymbol\Phi$ is exactly known and $\mathbf H$ is diagonal and positive (hence full-rank), we can always find a gain $\gamma>0$ to guarantee that the symmetric matrix
		\begin{equation}
		\mathbf C = 2\mathbf H - \gamma \mathbf H \boldsymbol\Phi \boldsymbol\Phi^\T \mathbf H > 0,
		\label{eq:learning_gain_condition}
		\end{equation}
		is also positive-definite, and therefore, full-rank.
		Next, let us re-arrange $mn$ linearly \emph{independent} row vectors from $\boldsymbol\Phi$ as follows:
		\begin{equation}
		\begin{bmatrix}
		\mathbf u_1^\T & \mathbf 0_n^\T & \cdots & \mathbf 0_n^\T \\
		\mathbf u_2^\T & \mathbf 0_n^\T & \cdots & \mathbf 0_n^\T \\
		\vdots & \vdots & \vdots & \vdots \\
		\mathbf u_n^\T & \mathbf 0_n^\T & \cdots & \mathbf 0_n^\T \\
		\mathbf 0_n^\T & \mathbf u_{n+1}^\T & \cdots & \mathbf 0_n^\T \\
		\mathbf 0_n^\T & \mathbf u_{n+2}^\T & \cdots & \mathbf 0_n^\T \\
		\vdots & \vdots & \ddots & \vdots \\
		\mathbf 0_n^\T & \mathbf 0_n^\T & \cdots & \mathbf u_{mn-1}^\T\\
		\mathbf 0_n^\T & \mathbf 0_n^\T & \cdots & \mathbf u_{mn}^\T
		\end{bmatrix}
		\end{equation}
		which shows that $\boldsymbol\Phi$ has a full column rank, hence, the matrix $\boldsymbol\Omega = \gamma\boldsymbol\Phi^\T \mathbf C \boldsymbol\Phi>0$ is positive-definite.
		This condition implies that $V_{t+1}^l - V_t^l < 0$ for any $\|\widehat{\mathbf a}_t^l - \mathbf a^l\|\ne 0$.
		Asymptotic stability of the parameter's estimation error directly follows by invoking Lyapunov's direct method \cite{Preprint:Bof2018}.
	\end{proof}
	
	\begin{remark}
		There are two conditions that need to be satisfied to ensure the algorithm's stability.
		The first condition is related to the magnitude of the learning gain $\gamma$. 
		Large gain values may lead to numerical instabilities, which is a common situation in discrete-time adaptive systems.
		To find a ``small enough'' gain $\gamma>0$, we can conduct the simple 1D search shown in Algorithm 1. An eigenvalue test on $\mathbf C$ can be used to verify \eqref{eq:learning_gain_condition}.
		The second condition is related to the linear independence (i.e. the non-collinearity) of the motor actions $\mathbf u_t$. 
		Such independent vectors are needed for providing a sufficient number of constraints to the estimation algorithm (this condition can be easily satisfied by performing random babbling-like motions).
	\end{remark}

	\subsection{Localised Adaptation}
	Once the cost function \eqref{eq:cost_function_distributed} has been minimised, the computed transformation matrix $\widehat{\mathbf A}_t$ \emph{locally approximates} the robot's sensorimotor model around the $l$th unit. 
	Note that the stability of the total $N$ units is analogous the analysis shown in the previous section; A global analysis is out of the scope of this work.

	\begin{algorithm}[t]
	\caption{{$\text{Compute a suitable }\gamma$}}
	\begin{algorithmic}[1]
	\State {$\gamma\gets$ initial value $<1$, $\mu\gets$ small step}
	\Repeat
	\State $\gamma\gets\gamma - \mu$
	\Until {$\mathbf C>0$}
	\end{algorithmic}
	\end{algorithm}

	The associated local training data \eqref{eq:data_structure} must then be released from memory to allow for new relations to be learnt---if needed.
	However, for the case where changes in the sensorimotor conditions occur, the model may contain inaccuracies in some or all computing units, and thus, its transformation matrices cannot be used for controlling the robot's motion.
	To cope with this issue, we need to first quantitatively assess such errors.
	For that, the following weighted \emph{distortion} metric is introduced:
	\begin{equation}
	U_t = \mathbf e_t^\T \mathbf B \mathbf e_t
	\label{eq:distortion_metric}
	\end{equation}
	where $\mathbf B>0$ denotes a positive-definite diagonal weight matrix to homogenise different scales in the approximation error $\mathbf e_t = \widehat{\mathbf A}^s \mathbf u_t - \boldsymbol\delta_t\in\mathbb R^m$.
	The scalar index $s$ is found by solving the search problem:
	\begin{equation}
	s = \argmin_j \|\mathbf w^j - \mathbf x_t\|
	\label{eq:search_problem}
	\end{equation}

	To enable adaptation of problematic units, we evaluate the magnitude of the metric $U_t$, and if found to be larger than an arbitrary threshold $U_t > |\varepsilon|$, new motion and sensor data must be collected around the $s$th computing unit to construct the revised structure $\mathcal D^s$ by using a \emph{push} approach:
	\begin{equation}
	\mathbf d_1 \leftarrow \begin{Bmatrix} \mathbf x_t & \mathbf u_t & \boldsymbol\delta_t \end{Bmatrix}
	\label{eq:babbling_collect_data}
	\end{equation}
	that updates the topmost observation and discards the oldest (bottom) data, so as to keep a constant number $\tau$ of data points.
	The transformation matrices are then computed with the new data.

	\subsection{Motion Controller}
	The update rule \eqref{eq:update_rule} computes an adaptive transformation matrix $\widehat{\mathbf A}_t^l$ for each of the $N$ units in the system.
	To provide a smooth transition between different units, let us introduce the matrix $\mathbf L_t\in\mathbb R^{m\times n}$ which is updated as follows\footnote{For simplicity, we initialise $\mathbf L_0 = \mathbf 0_{n\times n}$ with a zero matrix.}:
	\begin{equation}
	\mathbf L_{t+1} = \mathbf L_t - \eta \left( \mathbf L_t - \widehat{\mathbf A}_t^s \right)
	\end{equation}
	where $\eta>0$ is a tuning gain.
	The above matrix represents a filtered version of $\widehat{\mathbf A}_t^s$, where $s$ denotes the index of the active unit, as defined in \eqref{eq:search_problem}.
	With this approach, the transformation matrix smoothly changes between adjacent neighbourhoods, while providing stable values in the vicinity of the active unit; It can be seen as a continuous interpolation between adjacent neighbourhoods.
	
	The motor command with adaptive model is implemented as follows:
	\begin{equation}
	\mathbf u_t = - \lambda \mathbf L^{\#}_t \sat(\mathbf y_t-\mathbf y^*)
	\label{eq:adaptive_motor_action}
	\end{equation}
	The stability of this \emph{kinematic control} method can be analysed with its resulting closed-loop first-order system (a practice also commonly adopted with visual servoing controllers \cite{Journals:Chaumette2006}).
	To this end, we use a small displacement approach (motivated by the local target provided by the saturation function), where we introduce the increment vector $\mathbf i = - \sat(\mathbf y_t - \mathbf y^*)$ and define the local reference position $\overline{\boldsymbol y} = \mathbf y_t + \mathbf i \in\mathbb R^m$.
	Let us consider the case when the $N$ units have minimised the cost functions \eqref{eq:cost_function_distributed}.
	Note that the asymptotic minimisation of $\|\widehat{\mathbf a}_t^l - \mathbf a^l\|$ implies that $\widehat{\mathbf A}_t^s$ inherits the rank properties of $\mathbf A_t$, hence, the existence of the pseudo-inverse in \eqref{eq:adaptive_motor_action} is guaranteed; A regularisation term (see e.g. \cite{Book:Tikhonov2013}) can further be used to robustify the computation of $\mathbf L^{\#}_t$.
	
	\begin{proposition}
		For $n\ge m$ (i.e. more/equal motor actions than feedback features), the ``stiff'' kinematic control input \eqref{eq:adaptive_motor_action} provides the local feedback error $\mathbf y_t - \overline{\boldsymbol y}$ with asymptotic stability.
	\end{proposition}
	
	\begin{proof}
		Substitution of the controller \eqref{eq:adaptive_motor_action} into the difference model \eqref{eq:differential_model} yields the closed-loop system:	
		\begin{align}
		\mathbf y_{t+1} &= \mathbf y_t - \lambda \sat(\mathbf y_t - \mathbf y^*) = \mathbf y_t + \lambda \mathbf i \pm \lambda\mathbf y_t \nonumber \\
		&= \mathbf y_t - \lambda\mathbf y_t + \lambda \overline{\boldsymbol y} = \mathbf y_t - \lambda (\mathbf y_t - \overline{\boldsymbol y} ) \label{eq:small_displacement_stability}
		\end{align}
		Adding $\pm\overline{\boldsymbol y}$ to \eqref{eq:small_displacement_stability} and after some algebraic operation, we obtain:
		\begin{equation}
		\left(\mathbf y_{t+1} - \overline{\boldsymbol y} \right) = (1-\lambda)\left(\mathbf y_t - \overline{\boldsymbol y}\right)
		\end{equation}
		which for a gain satisfying $0<\lambda<1$, it implies local asymptotic stability of the small displacement error $(\mathbf y_t - \overline{\boldsymbol y})$ \cite{Book:Kuo_digital}.
	\end{proof}
	
	\begin{remark}
	Note that the above stability analysis assumes that robot's trajectories are not perturbed by external forces and that the estimated interaction matrix locally satisfies $\mathbf A_t \mathbf L^{\#}_t \mathbf A_t \approx \mathbf A_t $ around the active neighbourhood 
	\end{remark}

	\section{Case of Study}\label{sec:results}
	In this section, we validate the performance of the proposed method with numerical simulations and experiments.
	A vision-based manipulation task with a deformable cable is used as our case of study \cite{Journals:Bretl2014}: It consists in the robot actively deforming the object into a desired shape by using visual feedback of the cable's contour (see e.g. \cite{Proceedings:Zhu2018}).	
	Soft object manipulation tasks are challenging---and relevant to the fundamental problem addressed here---since the sensorimotor models of deformable objects are typically unknown or subject to large uncertainties \cite{Journals:Sanchez2018}. 
	Therefore, the transformation matrix relating the shape feature functional and the robot motions is difficult to compute.
	The proposed algorithm will be used to adaptively approximate the unknown model.
	Figure \ref{fig:cable_manipulation} conceptually depicts the setup of this sensorimotor control problem.

	\begin{figure}[t]
		\centering
		\includegraphics[scale = 1]{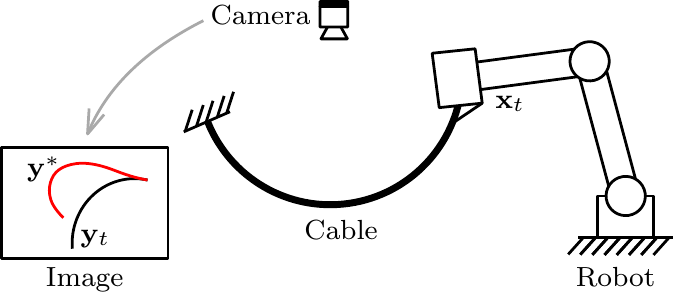}
		\caption{Representation of the cable manipulation case of study, where a vision sensor continuously measures the cable's feedback shape $\mathbf y_t$, which must be actively deformed towards $\mathbf y*$.}
		\label{fig:cable_manipulation}
	\end{figure}

	\subsection{Simulation Setup}
	For this study, we consider a planar robot arm that rigidly grasps one end of an elastic cable, whose other end is static; We assume that the total motion of this composed cable-robot system remains on the plane.
	A monocular vision sensor observes the manipulated cable and measures its 2D contour in real-time.
	The dynamic behaviour of the elastic cable is simulated as in \cite{Journals:Wakamatsu2004} by using the minimum energy principle \cite{Book:Hamill2014}, whose solution is computed using the CasADi framework \cite{Journals:Andersson2018}.
	The cable is assumed to have negligible plastic behaviour. 
	All numerical simulation algorithms are implemented in MATLAB. 
	The cable simulation code is publicly available at \url{https://github.com/Jihong-Zhu/cableModelling2D}.

	Let the long vector $\mathbf s_t\in\mathbb R^{2\alpha}$ represents the 2D profile of the cable, which is simulated using a resolution of $\alpha=100$ data points.
	To perform the task, we must compute a vector of feedback features $\mathbf y_t$ that characterises the object's configuration.
	For that, we use the approach described in \cite{Journals:Digumarti2019,dna2018_tro} that approximates $\mathbf s_t$ with truncated Fourier series (in our case, we used 4 harmonics), and then constructs $\mathbf y_t$ with the respective Fourier coefficients \cite{Proceedings:Collewet2000}.
	The use of these coefficients as feedback signals enable us to obtain a compact representation of the object's configuration, however, it complicates the analytical derivation of the matrix $\mathbf A_t$.
	
	\subsection{Approximation of $\mathbf A_t$}
	To construct the data structure \eqref{eq:data_structure}, we collect $\tau=40$ data observations $\mathbf d_t$ at random locations around the manipulation workspace.
	Next, we define local neighbourhoods centred at the configuration points $\mathbf w^1=[0.3,0.5]$, $\mathbf w^2=[0.5,0.5]$, $\mathbf w^3=[0.5,0.3]$ and $\mathbf w^4=[0.5,0.5]$.
	These neighbourhoods are defined with a standard deviation of $\sigma=1.3$.
	With the collected observations, $l=1,\ldots,4$ matrices $\widehat{\mathbf A}^l_t$ are computed using the update rule \eqref{eq:updae_rule_scalar}.

	\begin{figure}
		\centering
		\includegraphics[width=0.8\columnwidth]{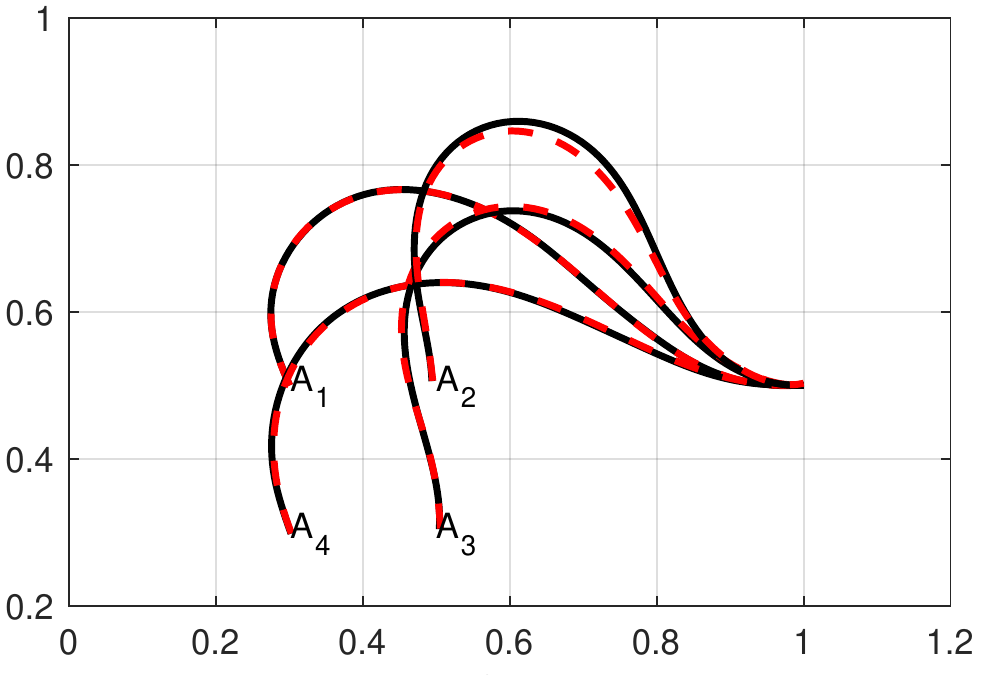}
		\caption{Various configurations of the visually measured cable profile (black solid line) and its approximation with Fourier series (red dashed line).}
		\label{fig:single_shape_test}
	\end{figure}

	Figure \ref{fig:single_shape_test} depicts the measured shape (black solid line) of the cable at the four points $\mathbf w^l$ and the shape that is approximated (red dashed line) with the feedback feature vector $\mathbf y_t$ (i.e the Fourier coefficients).
	It shows that 4 harmonics provide sufficient accuracy for representing the object's configuration. 
	To evaluate the accuracy of the computed discrete configuration space and its associated matrices $\widehat{\mathbf A}^l_t$, we conduct the following test: The robot is commanded to move the cable along a circular trajectory that passes through the four points $\mathbf w^l$. 
	The following energy function is computed throughout this trajectory:
	\begin{equation}
	G = \left\| \boldsymbol\delta_t-\widehat{\mathbf A}_t^l\mathbf u_t\right\|^2	
	\end{equation}
	which quantifies the accuracy of the local differential mapping \eqref{eq:differential_model}. 
	The index $l$ switches (based on the solution of \eqref{eq:search_problem}) as the robot enters a different neighbourhood.
	
	Figure \ref{fig:single_shape_test_cost} depicts the profile of the function $G$ along the trajectory. 
	We can see that this error function increases as the robot approaches the neighbourhood's boundary.
	The ``switch'' label indicates the time instant when ${\mathbf A}_t^l$ switches to different (more accurate) matrix, an action that decreases the magnitude of $G$.
	This result confirms that the proposed adaptive algorithm provides local directional information on how the motor actions transform into sensor changes.

	\begin{figure}
		\centering
		\includegraphics[width=\columnwidth]{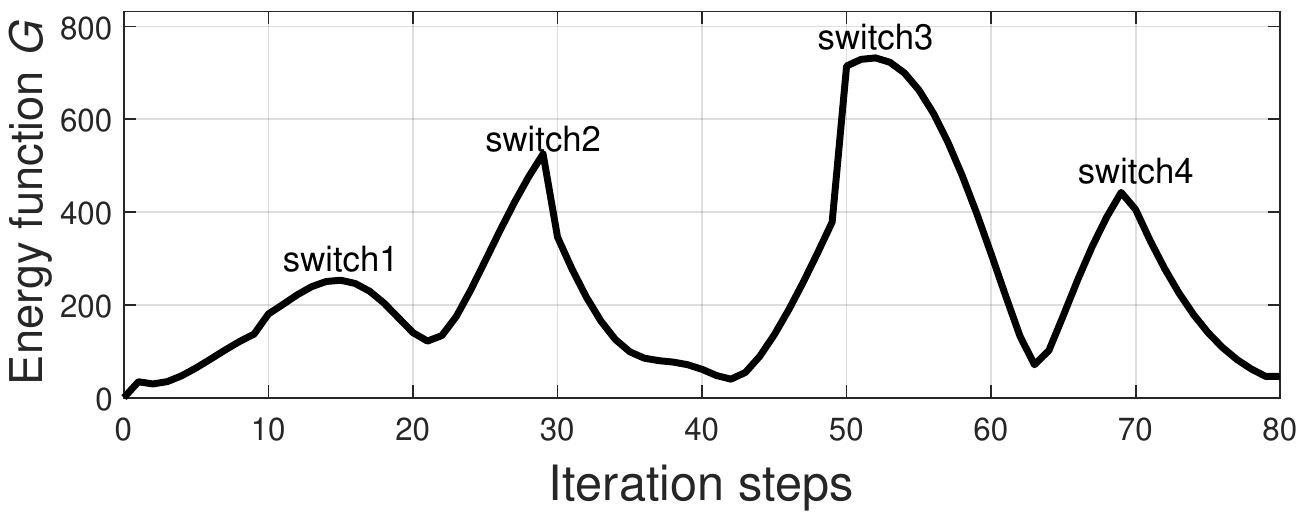}
		\caption{Profile of the function $G$ that is computed along the circular trajectory passing through the points in Figure \ref{fig:single_shape_test}; The ``switch'' label indicates the instant when $\widehat{\mathbf A}_t^l$ switches to different one.}
		\label{fig:single_shape_test_cost}
	\end{figure}

	\subsection{Sensor-Guided Motion}
	In this section, we make use of the approximated sensorimotor model to guide the motion of a robotic system based on feedback features.
	To this end, various cable shapes are defined as target configurations $\mathbf y^*$ (to provide physically feasible targets, these shapes are collected from previous sensor observations). 
	The target configurations are then given to the motion controller \eqref{eq:adaptive_motor_action} to automatically perform the task. 
	The controller implemented with saturation bounds of $|\sat(\cdot)|\le2$ and a feedback gain $\lambda=0.1$.

	Figure \ref{fig:trajectories} depicts the progression of the cable shapes obtained during these numerical simulations.
	The initial $\mathbf y_0$ and the intermediate configurations are represented with solid black curves, whereas the final shape $\mathbf y^*$ is represented with red dashed curves.
	To assess the accuracy of the controller, the following cost function is computed throughout the shaping motions:
	\begin{equation}
	E = \left\| \mathbf y_t - \mathbf y^* \right\|^2	
	\end{equation}
	For these four shaping actions, Figure \ref{fig:single_control_test_cost} depicts the time evolution of the function $E$.
	This figure clearly shows that the feedback error is asymptotically minimised.

	
	\begin{figure}[t]
		\centering
		\includegraphics[width=0.8\columnwidth]{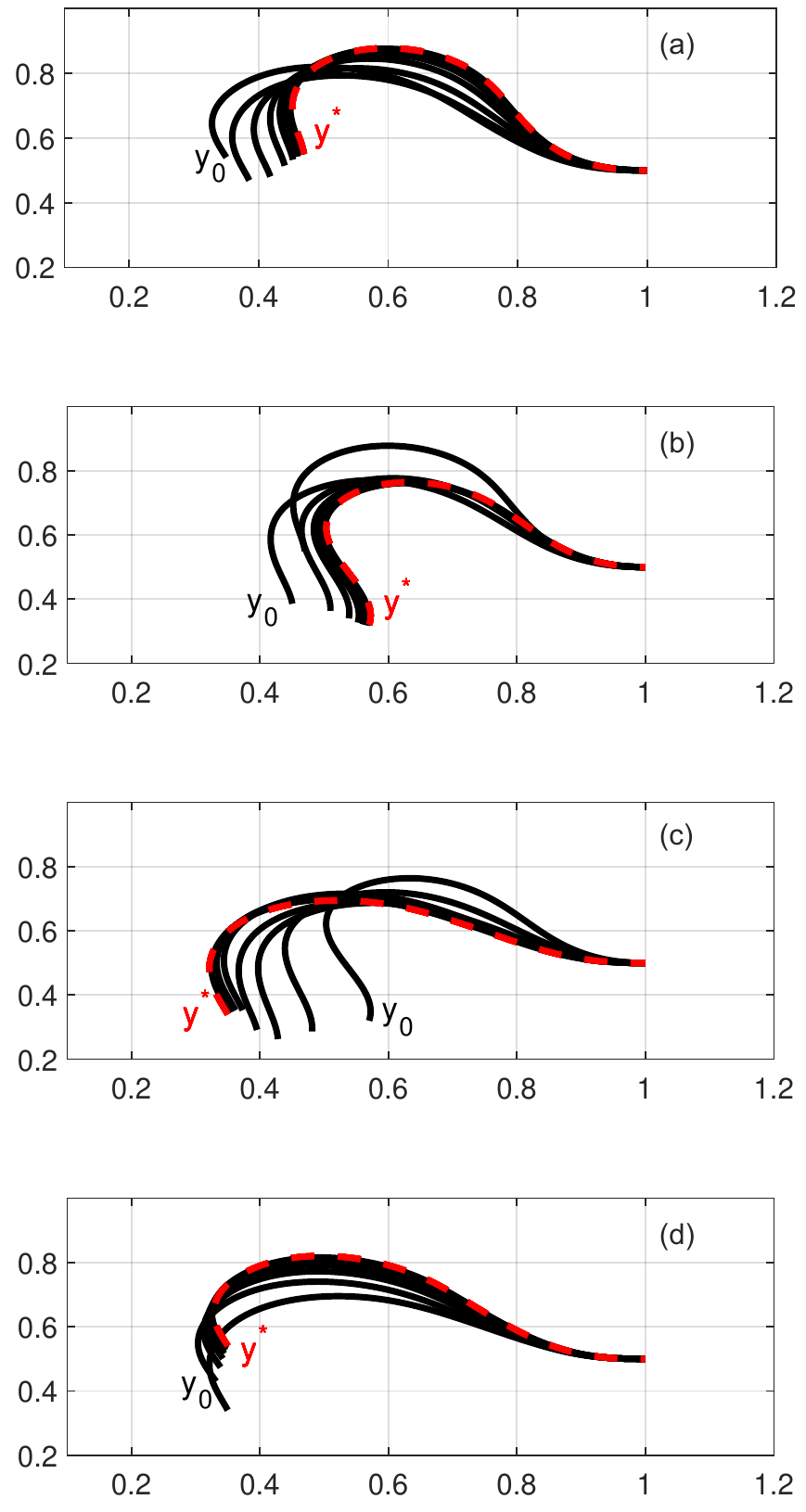}
		\caption{Initial and final configurations of the shape control simulation with a single robot}
		\label{fig:trajectories}
	\end{figure}
	
	\begin{figure}
		\centering
		\includegraphics[width=\columnwidth]{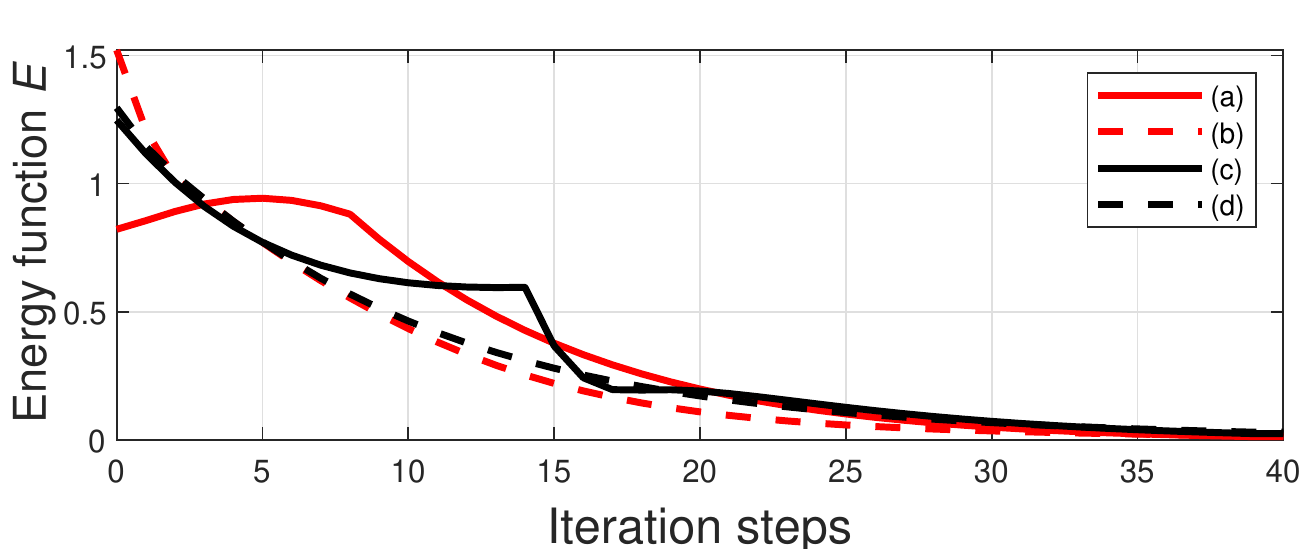}
		\caption{Minimization process of the energy function {$E$}}
		\label{fig:single_control_test_cost}
	\end{figure}
	 
	Now, consider the setup depicted in Figure \ref{fig:two_robot}, which has two 3-DOF robots jointly manipulating the deformable cable.
	For this more complex scenario, the total configuration vector $\mathbf x_t$ must be constructed with the 3-DOF pose (position and orientation) vectors of both robot manipulators as $\mathbf x_t = [\tensor*[^L]{\mathbf x}{_t^\T},\tensor*[^R]{\mathbf x}{_t^\T}{]}^\T\in\mathbb R^6$.
	Training of the sensorimotor model is done similarly as with the single-robot case described above; The same feedback gains and controller parameters are also used in this test.

	Figure \ref{fig:double_3_3} depicts the initial shape $\mathbf y_0$ and intermediate configurations (black solid curves), as well as the respective final shape $\mathbf y^*$ (red dashed curve) of the cable.
	Note that as more input DOF can be controlled by the robotic system, the object can be actively deformed into more complex configurations (cf. the achieved S-shape curve with the profiles in Figure \ref{fig:trajectories}).
	The result demonstrates that the approximated sensorimotor model provides sufficient directional information to the controller to properly ``steer'' the feature vector $\mathbf y_t$ towards the target $\mathbf y^*$.

	We now compare the performance of our method (using the same manipulation task shown in Figures \ref{fig:two_robot} and \ref{fig:double_3_3}) with two state-of-the-art approaches commonly used for guiding robots with unknown sensorimotor models.
	To this end, we consider the classical Broyden update rule \cite{Journals:Broyden:1965} and the recursive least-squares (RLS) \cite{Proceedings:Hosoda1994}.
	These two methods are used for estimating the matrix $\mathbf A$ that is needed to compute the control input \eqref{eq:general_motor_action}.
	To compare their performance, the cost function $E$ is evaluated throughout their respective trajectories; The same feedback gain $\lambda = 0.1$ is used for these three methods.
	Figure \ref{fig:double_3_3_cost} depicts the time evolution of $E$ computed with the three methods.
	This result demonstrates that the performance of our method is comparable to the other two classical approaches.

	\begin{figure}[t]
		\centering
		\includegraphics[scale = 1]{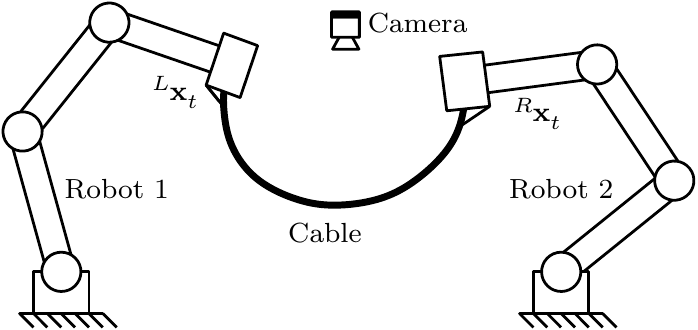}
		\caption{Representation of a two-robot setup where both systems must jointly shape the cable into a desired form.}
		\label{fig:two_robot}
	\end{figure}

	\begin{figure}
		\centering
		\includegraphics[width=0.8\columnwidth]{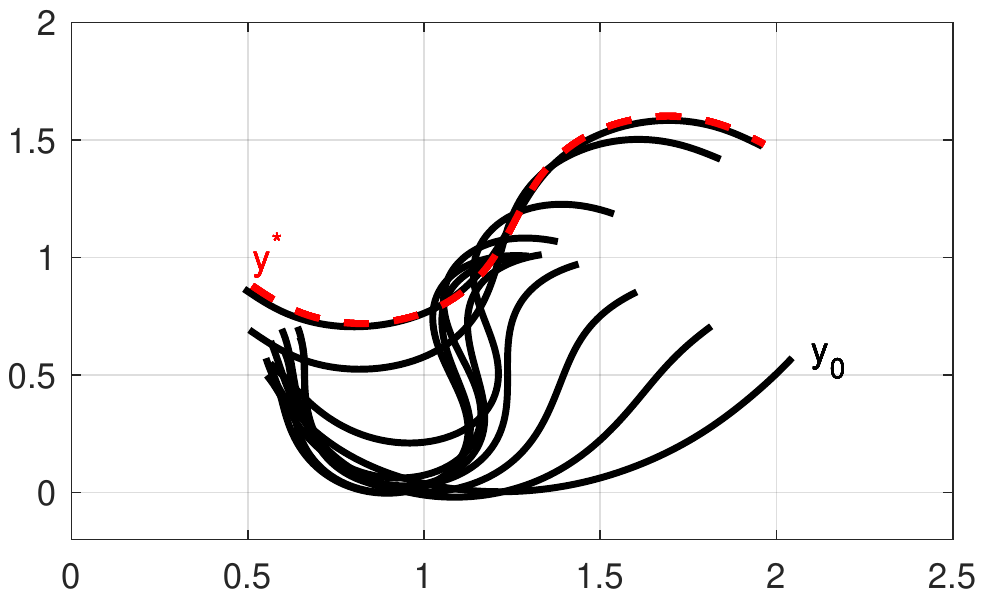}
		\caption{Initial and final configurations of the shape control simulation with two robots.}
		\label{fig:double_3_3}
	\end{figure}
	
	\begin{figure}[t]
		\centering
		\includegraphics[width=\columnwidth]{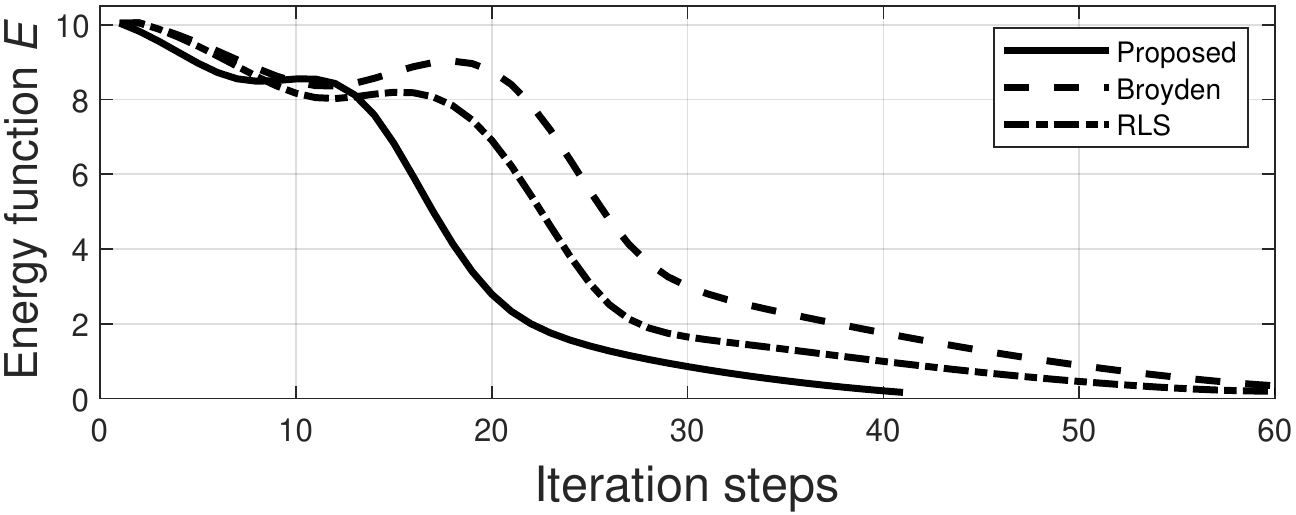}
		\caption{Minimization process of the energy function {$E$}.}
		\label{fig:double_3_3_cost}
	\end{figure}

	\subsection{Experiments}
	To validate the proposed theory, we developed an experimental platform composed of a three degrees-of-freedom serial robotic manipulator (DOBOT Magician), a Linux-based motion control system (Ubuntu 16.04), and a USB Webcam (Logitech C270); Image processing is performed by using the OpenCV libraries \cite{Journals:Bradski2000}.
	A sampling time of $\diff t \approx 0.04$ seconds is used in our Linux-based control system.
	In this setup, the robot rigidly grasps an elastic piece of pneumatic air tubing, whose other end is attached to the ground.
	The 3-DOF mechanism has a double parallelogram structure that enables to control the gripper's x-y-z position while keeping a constant orientation.
	For this experimental study, we only control 2-DOF of the robot such it manipulates the tubing with plane motions.
	Figure \ref{fig:dobot} depicts the setup.

	We conduct similar vision-guided experiments with the platform as the ones described in the previous section.
	For these tasks, the elastic tubing must be automatically positioned into a desired contour.
	The configuration dependant feedback for this task is computed with the observed contour of the object by using 2 harmonic terms \cite{dna2018_tro}.
	The sensorimotor model is similarly approximated around 4 configuration points (as in Figure \ref{fig:single_shape_test}), by performing random motions and collecting sensor data.

	Figure \ref{fig:experiments} depicts snapshots of the conducted experiments, where we can see the initial and final configurations of the system. 
	The red curves represent the (static) target configuration $\mathbf y^*$.
	For these two targets, Figure \ref{fig:experiments_minimisation} depicts the respective time evolution profiles of the energy function $E$, where we can clearly see that the feedback error is asymptotically minimised.
	The control inputs $\mathbf u_t$ used during the experiments are depicted in Figures \ref{fig:control1} and \ref{fig:control2}. 
	These motion commands are computed from raw vision measurements and a saturation threshold of $\pm 1$ is applied to its values.
	This results demonstrate that the approximated model can be used to locally guide motions of the robot with sensor feedback.

	\begin{figure}[t]
		\centering
		\includegraphics[width=\columnwidth]{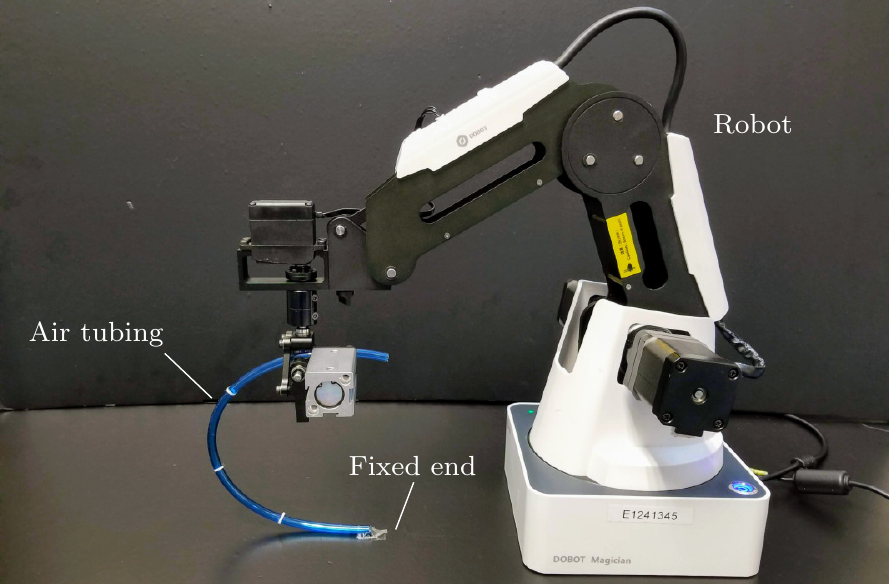}
		\caption{The experimental robotic setup}
		\label{fig:dobot}
	\end{figure}

	\begin{figure}
		\centering
		\includegraphics[width=\columnwidth]{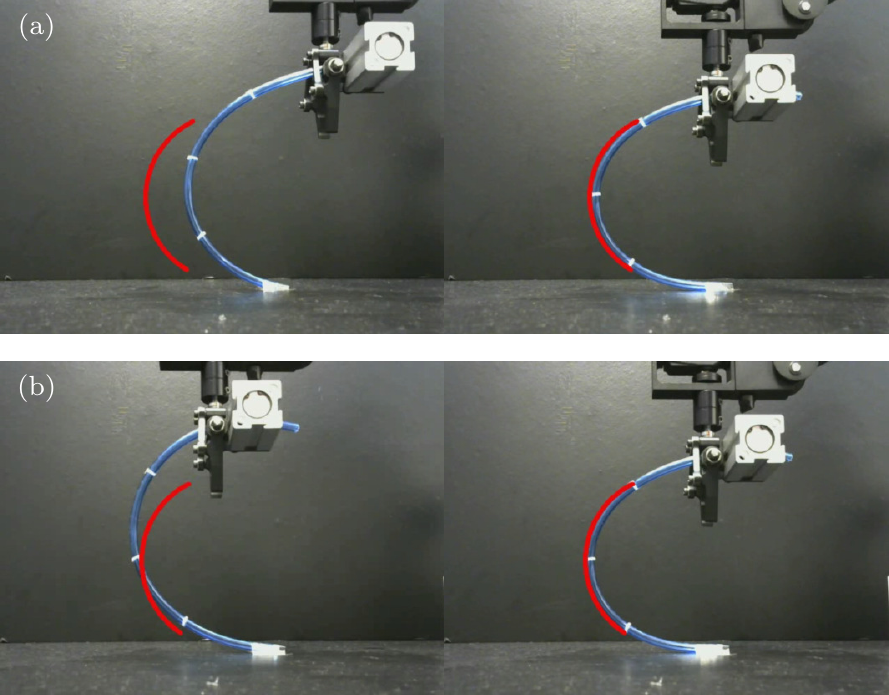}
		\caption{Snapshots of the initial (left image) and final (right image) configurations of the robot, where the red curve represents the target shape.}
		\label{fig:experiments}
	\end{figure}

	\begin{figure}
		\centering
		\includegraphics[width=\columnwidth]{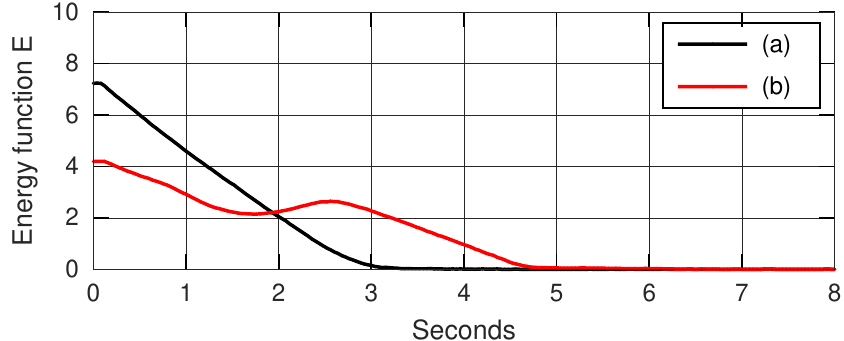}
		\caption{Asymptotic minimisation of the error functional $E$ obtained with the experiments shown in Figure \ref{fig:experiments}.}
		\label{fig:experiments_minimisation}
	\end{figure}
	
	\section{Conclusion}\label{sec:conclusion}
	In this paper, we describe a method to estimate sensorimotor relations of robotic systems.
	For that, we present a novel adaptive rule that computes local sensorimotor relations in real-time; The stability of this algorithm is rigorously analysed and its convergence conditions are derived.
	A motion controller to coordinate sensor measurements and robot motions is proposed.
	Simulation and experimental results with a cable manipulation case of study are reported to validate the theory.
	
	The main idea behind the proposed method is to divide the robot's configuration workspace into discrete nodes, and then, locally approximate at each node the mappings between robot motions and sensor changes.
	This approach resembles the estimation of piecewise linear systems, except that in our case, the computed model represents a differential Jacobian-like relation.
	The key guarantee the stability of the algorithm lies in collecting sufficient linear independent motor actions (such condition can be achieved by performing random babbling motions).
	
	The main limitation of the proposed algorithm is the local nature of its model, which can be improved by increasing the density of the distributed computing units.
	Another issue is related to the scalability of its discretised configuration space. 
	Note that for 3D spaces, the method can fairly well approximate the sensorimotor model, yet for multiple DOF (e.g. more than 6) the data is difficult to manage and visualise.

	As future work, we would like to implement our adaptive method with other sensing modalities and mechanical configurations, e.g. with an eye-in-hand visual servoing (where the camera orientation is arbitrary) and with variable morphology manipulators (where the link's length and joint's configuration are not known).

	\begin{figure}
		\centering
		\includegraphics[width=\columnwidth]{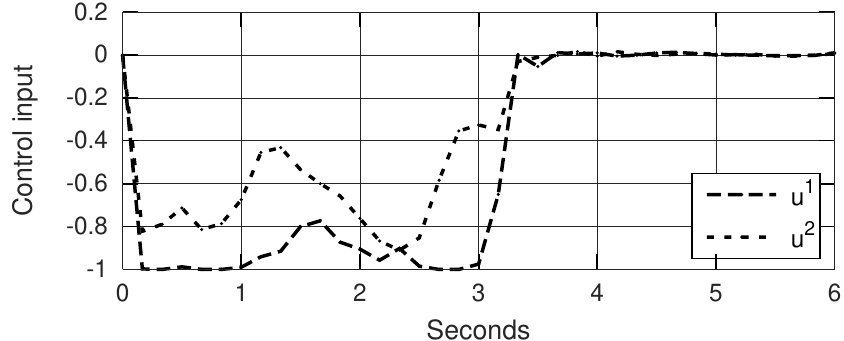}
		\caption{Control input (with normalised units of pixel/s) of the experiment (a) shown in Figure \ref{fig:experiments}.}
		\label{fig:control1}
	\end{figure}

	\begin{figure}
		\centering
		\includegraphics[width=\columnwidth]{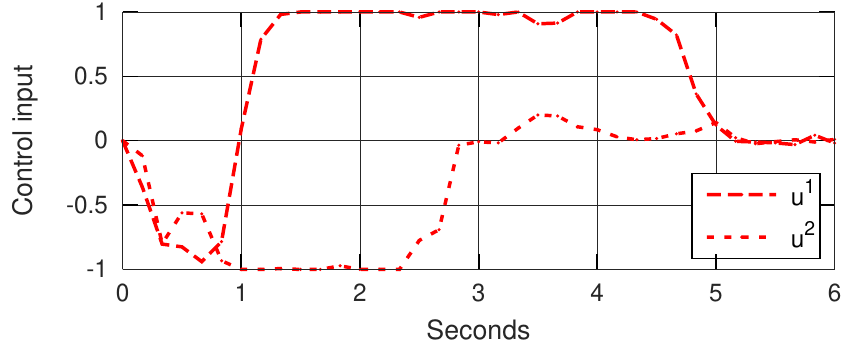}
		\caption{Control input (with normalised units of pixel/s) of the experiment (b) shown in Figure \ref{fig:experiments}.}
		\label{fig:control2}
	\end{figure}

	\section*{Funding}
	This research work is supported in part by the Research Grants Council (RGC) of Hong Kong under grant number 14203917, in part by PROCORE-France/Hong Kong Joint Research Scheme sponsored by the RGC and the Consulate General of France in Hong Kong under grant F-PolyU503/18, in part by the Chinese National Engineering Research Centre for Steel Construction (Hong Kong Branch) at PolyU under grant BBV8, in part by the Key-Area Research and Development Program of Guangdong Province 2020 under project 76 and in part by The Hong Kong Polytechnic University under grant G-YBYT.

\bibliographystyle{IEEEtran}
\bibliography{/home/dauid/Articles/biblio/david.bib,/home/dauid/Articles/biblio/bibliography.bib}

\end{document}